\documentclass[runningheads]{llncs}
\usepackage[T1]{fontenc}

\usepackage{amsmath,amssymb,amsfonts}
\usepackage{placeins}
\usepackage{graphicx}
\usepackage{textcomp}
\usepackage{xcolor}
\usepackage{newfloat}
\usepackage{listings}
\usepackage{booktabs}
\usepackage{enumitem}
\usepackage{cuted}

\providecommand{\Description}[1]{}
\providecommand{\citet}[1]{\cite{#1}}
\providecommand{\citep}[1]{\cite{#1}}

\def\BibTeX{{\rm B\kern-.05em{\sc i\kern-.025em b}\kern-.08em
    T\kern-.1667em\lower.7ex\hbox{E}\kern-.125emX}}
\usepackage[hidelinks]{hyperref}
\usepackage{xurl}
\begin{document}

\title{AGPO: Adaptive Group Policy Optimization with Dual Statistical Feedback}

% === author start ===
\newcommand{\orcidAuthorOne}{\href{https://orcid.org/0000-0002-6417-3280}{\protect\includegraphics[scale=0.045]{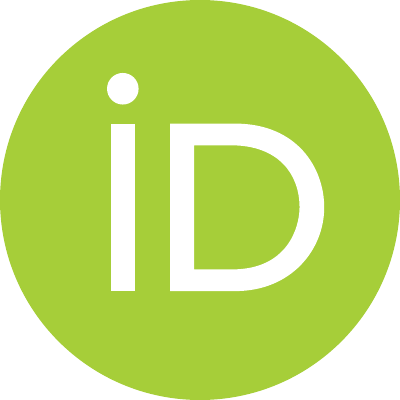}}}
\newcommand{\orcidAuthorTwo}{\href{https://orcid.org/0009-0007-5519-6222}{\protect\includegraphics[scale=0.045]{pic/orcid.pdf}}}
\newcommand{\orcidAuthorThree}{\href{https://orcid.org/0009-0008-0501-8652}{\protect\includegraphics[scale=0.045]{pic/orcid.pdf}}}
\newcommand{\orcidAuthorFour}{\href{https://orcid.org/0009-0000-9495-0389} 
{\protect\includegraphics[scale=0.045]{pic/orcid.pdf}}}
\newcommand{\orcidAuthorFive}{\href{https://orcid.org/0000-0002-1069-4830}{\protect\includegraphics[scale=0.045]{pic/orcid.pdf}}}
\newcommand{\orcidAuthorSix}{\href{https://orcid.org/0009-0007-2672-6650}{\protect\includegraphics[scale=0.045]{pic/orcid.pdf}}}
\newcommand{\orcidAuthorSeven}{\href{https://orcid.org/0009-0002-6042-3454}{\protect\includegraphics[scale=0.045]{pic/orcid.pdf}}}
\newcommand{\orcidAuthorEight}{\href{https://orcid.org/0000-0002-1799-3948}{\protect\includegraphics[scale=0.045]{pic/orcid.pdf}}}

\author{
Miaobo Hu\inst{1,2}\orcidAuthorOne \and
Bokun Wang\inst{1,2}\orcidAuthorTwo \and
Shuhao Hu\inst{1,2}\orcidAuthorThree \and
Ruohan Wang\inst{1,2}\orcidAuthorFour \and
Xin Wang\inst{1}\orcidAuthorFive \and
Xiaobo Guo\inst{1}\orcidAuthorSix \thanks{Corresponding author: guoxiaobo@iie.ac.cn} \and
Daren Zha\inst{1}\orcidAuthorSeven \and
Jun Xiao\inst{3}\orcidAuthorEight
}

\authorrunning{M. Hu et al.}

\institute{
Institute of Information Engineering, Chinese Academy of Sciences, Beijing 100092, China
\email{\{humiaobo,wangbokun,hushuhao,wangruohan\}@iie.ac.cn}
\email{\{wangxin,guoxiaobo,zhadaren\}@iie.ac.cn}
\and
School of Cyber Security, University of Chinese Academy of Sciences, Beijing 100049, China
\and
School of Artificial Intelligence, University of Chinese Academy of Sciences, Beijing 100049, China
\email{xiaojun@ucas.ac.cn}
}

% === author end ===

\maketitle

\begin{abstract}
Reinforcement learning improves LLM reasoning, but PPO/GRPO typically use fixed clipping and decoding temperature, which makes training brittle and tuning-heavy. We propose Adaptive Group Policy Optimization (AGPO), a critic-free refinement of GRPO that uses group-level statistics to control both update magnitude and exploration. AGPO uses a shared probe-derived statistical state to drive two controllers:
(i) adaptive clipping, which sets the trust-region size from reward
dispersion and skewness, probe vote entropy, policy entropy, and step-wise
KL drift; and (ii) bidirectional adaptive temperature sampling, which heats
or cools decoding around a base temperature according to centered uncertainty
relative to a running baseline. On nine English and Chinese math/STEM benchmarks, Qwen2.5-14B trained with AGPO outperforms PPO/GRPO under the same generated-token budget, reaching 67.3\% on GSM8K and 40.5\% on MATH. Gains transfer to Llama-3-8B and Gemma-2-9B, and ablations confirm both modules are complementary. 
Our implementation is publicly available at \url{https://github.com/wandugu/paper_agpo}.

\keywords{Reinforcement Learning \and Large Language Models \and Adaptive Clipping \and Temperature Sampling \and RLHF}
\end{abstract}

\section{Introduction}
\label{sec:intro}

Reinforcement learning (RL) improves LLM reasoning, but PPO and GRPO
usually keep the clipping radius $\varepsilon$ and decoding temperature
$\tau$ fixed \cite{schulman_proximal_2017,ouyang_training_2022,shao_deepseekmath_2024}.
In practice, reward dispersion, answer disagreement, policy entropy,
and KL drift vary substantially during training, so a single fixed
setting can under-update early and destabilize late-stage optimization.

This creates three recurring issues in RLHF: heavy tuning of trust-region
and decoding hyper-parameters, exploration--exploitation mismatch, and
brittle convergence. AGPO addresses them with observable group-level
statistics that control both update magnitude and exploration.

We propose \textbf{Adaptive Group Policy Optimization (AGPO)}, a
light-weight critic-free refinement of GRPO that uses group-level
statistics to adapt (i) the clipping radius and (ii) the sampling
temperature. As in GRPO, a \emph{group} denotes the $G$ rollouts
sampled for one prompt from a single policy.

\paragraph{Contributions.}
\begin{enumerate}[leftmargin=1.2em,itemsep=2pt]
  \item \textbf{Adaptive clipping.}
  We replace fixed GRPO clipping with a statistics-aware
  $\varepsilon_{\text{adaptive}}$ driven by reward dispersion/skewness,
  probe disagreement, policy entropy, and step-wise KL drift.

  \item \textbf{Shared-probe dual control.}
  One probe phase drives both adaptive clipping and Adaptive Temperature
  Sampling (ATS), improving the exploration/stability trade-off without
  a value critic.

  \item \textbf{Consistent gains.}
  On nine English and Chinese math/STEM benchmarks, AGPO improves over
  PPO/GRPO under equal generated-token budgets, and the gains transfer
  to Llama-3-8B and Gemma-2-9B.
\end{enumerate}

\section{Related Work}
\label{sec:related}

\paragraph{Critic-based vs.\ critic-free RLHF.}
PPO-style RLHF couples policy optimization with a learned value critic
\cite{schulman_proximal_2017,ouyang_training_2022}. Critic-free
alternatives such as DPO \cite{rafailov_direct_2024} and GRPO
\cite{shao_deepseekmath_2024} simplify training and reduce memory, but
typically keep clipping and sampling hyper-parameters fixed.

\paragraph{Adaptive strategies for LLM RL.}
Adaptive KL control has been used in LLM alignment
\cite{ouyang_training_2022}. AGPO differs by jointly adapting clipping
and sampling inside a critic-free loop from a shared group-level probe
state, targeting reasoning benchmarks where exploration and stability
both matter \cite{hendrycks_measuring_2021}.

\section{Method}
\label{sec:method}
\paragraph{Notation.}
We denote token-level Shannon entropy by $H(\pi)$ and the training horizon by $T$.
We use $D_{\mathrm{KL}}(\pi_{\theta_{\mathrm{old}}}\Vert\pi_{\theta_{\mathrm{new}}})$ as a step-wise drift signal and
$D_{\mathrm{KL}}(\pi_{\theta}\Vert\pi_{\mathrm{ref}})$ as a reference-anchoring regularizer.
For a prompt $x$, grouped rewards are $\{r_i\}_{i=1}^{G}$ with mean $\mu_r$ and dispersion estimator $\hat{\sigma}$;
temperatures and clip radii are bounded in $[\tau_{\min},\tau_{\max}]$ and $[\varepsilon_{\min},\varepsilon_{\max}]$, respectively.
We restate only the GRPO surrogate needed to define AGPO in Sect.~\ref{sec:ppo_grpo}.

\subsection{Adaptive Group Policy Optimization (AGPO)}
At a high level, AGPO is built around \emph{one probe, two controllers}. A single probe phase extracts a shared statistical state from reward dispersion, reward asymmetry, and answer-level disagreement. ATS centers this uncertainty against a running baseline and adjusts sampling temperature above or below $\tau_{\text{base}}$ to control rollout diversity, while adaptive clipping augments the same probe statistics with policy entropy and step-wise KL drift to set $\varepsilon_{\mathrm{adaptive}}$ for stable optimization. This shared-probe design---one statistical state jointly controlling exploration and update magnitude---is the central novelty of AGPO.

\subsubsection{Reinforcement Learning with Grouped Rollouts}

After supervised fine-tuning (SFT), the policy is refined by RL.  
For each input~$x$ we sample a \emph{group} of $G$ rollouts $\{e_i\}_{i=1}^{G}$ (default $G{=}8$).  
Here, a ``group'' always refers to multiple rollouts of a \emph{single} policy for the same prompt, rather than a team of distinct agents as in multi-agent RL.  
Each candidate answer~$\hat y_i$ receives a reward $r_i\in[0,1]$, e.g.\ exact-match or partial credit.  
Group-normalized advantage
\begin{equation}
\label{eq:adv}
\begin{aligned}
A_i &= \frac{r_i-\mu_r}{\hat{\sigma}+\epsilon_A},\\
\mu_r &= \frac{1}{G}\sum_{j=1}^{G} r_j,\qquad
\hat{\sigma} = \sqrt{\frac{1}{G-1}\sum_{j=1}^{G}(r_j-\mu_r)^2}\ \text{(default)}.
\end{aligned}
\end{equation}
\noindent\textit{We set $\epsilon_A=10^{-8}$ in all experiments.}
This normalisation scales updates by inter-candidate diversity.
In implementation, the rollout-level advantage $A_i$ is broadcast to the generated tokens of rollout $e_i$ when forming the token-level loss. Unless otherwise noted, $\hat{\sigma}$, $|\tilde{\kappa}_3(r)|$, and $\mathcal{E}_{\text{probe}}$ are computed per prompt group and then averaged over the update minibatch before entering the controllers; $H(\pi)$ and $\widehat{D}_{\mathrm{KL}}^{\mathrm{step}}$ are batch-level averages by definition.

\subsubsection{GRPO surrogate objective}
\label{sec:ppo_grpo}

AGPO builds on the clipped GRPO surrogate, a sequence-level analogue of PPO that uses group-normalized advantages and optionally a KL penalty to a reference policy. For a prompt $x$ and rollout $e_i$,
\begin{equation}
\rho_i(\theta)=\frac{\pi_\theta(e_i \mid x)}
{\pi_{\theta_{\text{old}}}(e_i \mid x)}.
\end{equation}

\begin{equation}
\label{eq:grpo_app}
\begin{aligned}
J_{\text{GRPO}}(\theta)
&=
\mathbb{E}_{x,\,\{e_i\}_{i=1}^{G}}
\Biggl[
\frac{1}{G}\sum_{i=1}^{G}
\Bigl(
\min\Bigl(
\rho_i(\theta)\,A_i,\\
&\qquad\qquad
\operatorname{clip}\bigl(\rho_i(\theta),1-\varepsilon,1+\varepsilon\bigr)A_i
\Bigr)
-\beta\,D_{\mathrm{KL}}\!\bigl(\pi_\theta\parallel\pi_{\text{ref}}\bigr)
\Bigr)
\Biggr].
\end{aligned}
\end{equation}

When $G=1$, Eq.~\eqref{eq:grpo_app} reduces to a PPO-style clipped surrogate. AGPO keeps the same structure and replaces the fixed clip radius $\varepsilon$ with an adaptive controller. We use two KL signals for different purposes: $D_{\mathrm{KL}}(\pi_{\theta_{\text{old}}}\parallel\pi_{\theta_{\text{new}}})$ serves as a detached step-wise drift signal, while $D_{\mathrm{KL}}(\pi_\theta\parallel\pi_{\text{ref}})$ is the reference regularizer in the optimization objective.

\subsubsection{Adaptive Clipping}

Fixed clipping ignores evolving batch uncertainty. AGPO therefore adapts the GRPO clip radius using reward dispersion, reward asymmetry, probe disagreement, policy entropy, and step-wise KL drift:
\begin{equation}
\begin{split}
\varepsilon_{\mathrm{adaptive}}
&= \operatorname{clip}\Bigl(
\varepsilon_{\mathrm{base}}\,
\frac{H(\pi_{\theta_{\mathrm{old}}})\bigl(1+\delta\,\mathcal{E}_{\mathrm{probe}}\bigr)}
{1+\alpha\,\hat{\sigma}+\zeta\,|\tilde{\kappa}_3(r)|
+\gamma\,\widehat{D}_{\mathrm{KL}}^{\mathrm{step}}},
\varepsilon_{\min},\,\varepsilon_{\max}
\Bigr).
\end{split}
\label{eq:eps_adapt}
\end{equation}
Because Eq.~\eqref{eq:step_kl_est} depends on the post-update policy, we use the measured $\widehat{D}_{\mathrm{KL}}^{\mathrm{step}}$ from the previous update (optionally EMA-smoothed) when computing $\varepsilon_{\mathrm{adaptive}}$ at the current step. This avoids circular dependence between the controller and the policy update.

Here $\hat{\sigma}$ is the within-group reward dispersion (std by default, with MAD/IQR variants discussed in Sect.~\ref{sec:robust_ablation}); $\tilde{\kappa}_3(r)$ is the safeguarded reward skewness; $\mathcal{E}_{\text{probe}}$ is the probe vote entropy; and $\widehat{D}_{\mathrm{KL}}^{\mathrm{step}}$ is the detached step-wise KL estimate from Eq.~\eqref{eq:step_kl_est}. Coefficients $\alpha,\gamma,\delta,\zeta\ge 0$ control sensitivity, and $[\varepsilon_{\min},\varepsilon_{\max}]$ bounds the controller.

Intuitively, higher policy entropy and probe disagreement enlarge the admissible update region when the batch appears uncertain, whereas larger reward dispersion, tail asymmetry, or step-wise KL drift shrink the region to prevent unstable policy jumps. We use these signals asymmetrically on purpose: probe disagreement and policy entropy reflect epistemic uncertainty and under-exploration, whereas reward dispersion, skewness, and step-wise KL drift indicate optimization risk and therefore tighten the trust region.

\begin{equation}
\tilde{\kappa}_3(r)
= \mathbb{I}[\hat{\sigma}\ge\sigma_{\min}]\,
\operatorname{clip}\!\bigl(\kappa_3(r),-\kappa_{\max},\kappa_{\max}\bigr),
\label{eq:skew_guard}
\end{equation}
where we use $\sigma_{\min}=10^{-6}$ and $\kappa_{\max}=10$.

\begin{equation}
\widehat{D}_{\mathrm{KL}}^{\mathrm{step}}
=
\frac{1}{|\mathcal{T}|}\sum_{(s_t,a_t)\in\mathcal{T}}
\Bigl(\log \pi_{\theta_{\mathrm{old}}}(a_t\mid s_t)
      -\log \pi_{\theta_{\mathrm{new}}}(a_t\mid s_t)\Bigr),
\label{eq:step_kl_est}
\end{equation}
where $\mathcal{T}$ is the set of sampled tokens in the minibatch. In practice, $\varepsilon_{\mathrm{adaptive}}$ is computed once per update from detached controller statistics and treated as a constant in Eq.~\eqref{eq:agpo}.

\subsubsection{AGPO Objective}
Substituting $\varepsilon_{\mathrm{adaptive}}$ into the GRPO objective yields the AGPO loss:
\begin{equation}
\label{eq:agpo}
\begin{aligned}
\mathcal{L}_{\text{AGPO}}(\theta)
&=
- \mathbb{E}_{x,\,\{e_i\}_{i=1}^{G}}
\Biggl[
\frac{1}{G}\sum_{i=1}^{G}
\Bigl(
\min\Bigl(
\rho_i(\theta)\,A_i,\\
&\qquad\qquad
\operatorname{clip}\bigl(
\rho_i(\theta),
1-\varepsilon_{\mathrm{adaptive}},
1+\varepsilon_{\mathrm{adaptive}}
\bigr)A_i
\Bigr)\\
&\qquad\qquad
-\beta\,D_{\mathrm{KL}}\!\bigl(\pi_\theta \,\|\, \pi_{\text{ref}}\bigr)
\Bigr)
\Biggr].
\end{aligned}
\end{equation}

Overall, AGPO couples an update controller ($\varepsilon_{\mathrm{adaptive}}$) with a rollout controller ($\tau_t$): the former regulates policy-step magnitude, while the latter regulates exploration during trajectory collection.

\begin{figure}[tb]
    \centering
    \includegraphics[width=0.95\linewidth]{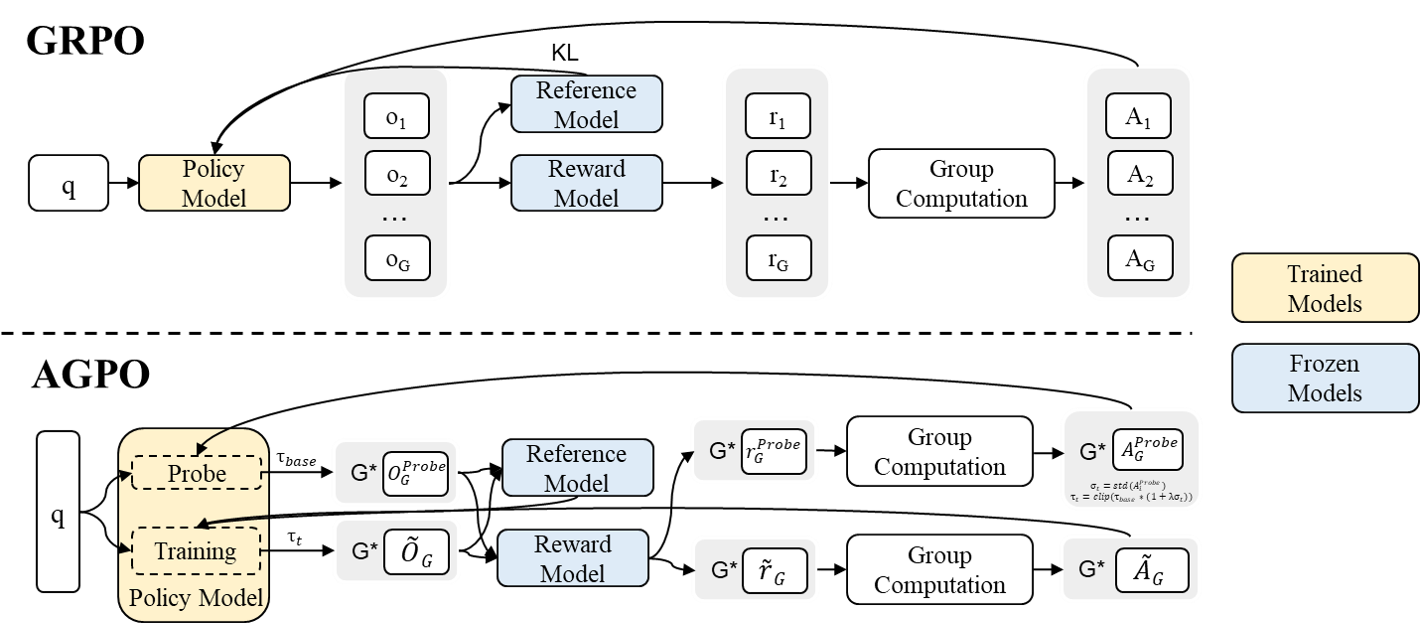}
    \caption{
    Two-phase AGPO with ATS. A \textbf{probe} at $\tau_{\text{base}}$ estimates group statistics (including reward dispersion $\hat{\sigma}$); a \textbf{train} phase then uses $\tau_t$ and the adaptive clip $\varepsilon_{\text{adaptive}}$ to update the policy via Eq.~\eqref{eq:agpo}. This coupling yields exploration when uncertain and stable updates otherwise.
    }
    \Description{Block diagram of AGPO with two phases: a probe at base temperature collects statistics such as reward dispersion and vote entropy; then a training phase uses an adaptive temperature and an adaptive clipping radius to update the policy via the AGPO loss. Arrows indicate the coupling between statistics and controllers.}
    \label{fig:agpo}
\end{figure}

\subsection{Adaptive Temperature Sampling Strategy}

Although advantage-based objectives guide \emph{updates}, the \emph{trajectories} collected for those updates hinge on the sampling temperature.  
Fixed temperatures ignore the changing reward landscape, so we introduce \textbf{Adaptive Temperature Sampling (ATS)}.

\subsubsection{Uncertainty--temperature coupling}
For the current minibatch we first run a probe at $\tau_{\text{base}}$ and compute
\begin{align}
\mathcal{E}_{\text{probe}}
  &= - \sum_{y} p_{\text{probe}}(y)\,\log p_{\text{probe}}(y),
\label{eq:probe_entropy}
\\[-2pt]
U_t &= w_r\,\hat{\sigma} \;+\; w_e\,\mathcal{E}_{\text{probe}} \;+\; w_k\,|\tilde{\kappa}_3(r)|,
\label{eq:uncert_combo}
\\[-2pt]
\widetilde U_t &= U_t - \bar U_t,
\label{eq:uncert_center}
\end{align}
where $p_{\text{probe}}(y)$ is the empirical distribution of extracted answers across probe rollouts, $\hat{\sigma}$ is the same within-group dispersion estimator as in Eq.~\eqref{eq:adv} (std/MAD/IQR) computed on probe rewards, $\tilde{\kappa}_3(r)$ is the safeguarded skewness defined in Eq.~\eqref{eq:skew_guard}, and $\bar U_t$ is the exponential moving average (EMA) of raw uncertainty scores from previous probe batches.

We maintain the running baseline by an exponential moving average
\begin{equation}
\bar U_t = \rho_U \bar U_{t-1} + (1-\rho_U) U_{t-1},
\end{equation}
where $\rho_U \in [0,1)$ is the EMA momentum. We initialize $\bar U_0=0$ and update it using probe batches only.

We then set
\begin{equation}
\tau_t = \operatorname{clip}\!\Bigl(\tau_{\text{base}}\bigl(1+\lambda\,\widetilde U_t\bigr),\; \tau_{\min},\tau_{\max}\Bigr),
\label{eq:tau_adapt}
\end{equation}
with non-negative weights $w_r,w_e,w_k$.
The coefficient $\lambda$ scales sensitivity. When $\widetilde U_t>0$, the batch is more uncertain than the running baseline and AGPO heats the sampling distribution; when $\widetilde U_t<0$, the batch is more certain than the baseline and AGPO cools the policy for exploitation.

While Eq.~\eqref{eq:tau_adapt} is heuristic, the raw score $U_t$ couples exploration strength to three orthogonal signals: (i) \emph{reward dispersion} $\hat{\sigma}$, (ii) \emph{answer-level disagreement} $\mathcal{E}_{\text{probe}}$, and (iii) \emph{tail asymmetry} $|\tilde{\kappa}_3(r)|$. Centering by $\bar U_t$ allows $\tau_t$ to move on both sides of $\tau_{\text{base}}$, so AGPO explores more on unusually uncertain batches and cools easier batches for exploitation \cite{holtzman_curious_2020}.

\subsection{AGPO Training Procedure}
\label{sec:impl_notes}

AGPO uses a probe--train loop (Fig.~\ref{fig:agpo}). For each update,
we sample one probe group per prompt at $\tau_{\text{base}}$, compute
group-level statistics, and average them into an update-level controller
state. We then form $U_t$, center it by $\bar U_t$, set scalar
$\tau_t$ and $\varepsilon_{\mathrm{adaptive}}$, resample training
rollouts at $\tau_t$, update the policy with Eq.~\eqref{eq:agpo}, and
carry the measured $\widehat{D}_{\mathrm{KL}}^{\mathrm{step}}$ to the
next step. Probe rollouts are used only for controller estimation and
are never reused for gradients.

\subsection{Finite-Horizon Convergence Sketch}
\label{sec:proof}

We provide a finite-horizon proof sketch for a token-level proxy of the clipped-surrogate component of AGPO, aligned with the autoregressive implementation used in practice. The result should therefore be read as an idealized guarantee for the practical optimizer rather than a full proof for the exact sequence-level objective in Eq.~\eqref{eq:agpo}:
\begin{equation}
\label{eq:agpo_surr}
\mathcal{L}^{\operatorname{clip}}_{\mathrm{AGPO}}(\theta)
= \mathbb{E}_{\mathcal{B}}\!\left[\rho_\theta(s,a)\,A_{\text{group}}(s,a)\right],
\end{equation}
where $A_{\text{group}}(s,a)$ denotes the group-normalized rollout advantage in Eq.~\eqref{eq:adv} broadcast to the generated tokens, and
\begin{equation}
\label{eq:rho_def}
\rho_\theta
=\operatorname{clip}\!\left(
\frac{\pi_\theta(a\mid s)}{\pi_{\theta_{\mathrm{old}}}(a\mid s)},
1-\varepsilon_{\text{adaptive}},\,
1+\varepsilon_{\text{adaptive}}
\right).
\end{equation}

Assume bounded rewards, a Lipschitz policy, and a finite-variance
stochastic gradient estimator satisfying
$\mathbb{E}[\|\hat g_t\|_2^2] \le C_g$.
For the idealized finite-horizon analysis we use
$\eta_t=\eta_0/\sqrt{T}$.

\begin{lemma}[Bounded controller]
Because $\varepsilon_{\text{adaptive}}$ and $\tau_t$ are clipped by construction, we have $\varepsilon_{\min}\le \varepsilon_{\text{adaptive}}\le \varepsilon_{\max}$ and $\tau_t\in[\tau_{\min},\tau_{\max}]$ for every update.
\end{lemma}

\begin{theorem}[Stationary-point convergence]
Under the assumptions above,
\begin{equation}
\frac{1}{T}\sum_{t=1}^{T}
\mathbb{E}\!\left[
\left\|\nabla_\theta
\mathcal{L}^{\operatorname{clip}}_{\mathrm{AGPO}}(\theta_t)\right\|_2^{2}
\right]
\le \frac{C}{\sqrt{T}},
\label{eq:stationary_rate_main}
\end{equation}
for a constant $C>0$ independent of $T$.
\end{theorem}

\begin{proof}[Sketch]
The bounded controller keeps $\varepsilon_{\text{adaptive}}$ and $\tau_t$ uniformly controlled, so the induced perturbation can be absorbed into the smoothness constants. Standard stochastic-approximation analysis then yields a descent inequality of the form
\[
\mathbb{E}[V_{t+1}\mid\mathcal{F}_t]
\le
V_t
-\eta_t\|\nabla \mathcal{L}^{\operatorname{clip}}_{\mathrm{AGPO}}(\theta_t)\|_2^2
+\eta_t^2 C_g
\]
for a suitable Lyapunov function $V_t$. Summing over $t=1,\dots,T$ and using $\eta_t=\eta_0/\sqrt{T}$ gives Eq.~\eqref{eq:stationary_rate_main}.
\end{proof}

\paragraph{Remark.}
This theorem is an idealized guarantee for the clipped-surrogate
controller. Our experiments use cosine learning-rate decay and a fixed
temperature-sensitivity coefficient $\lambda$ in
Eq.~\eqref{eq:tau_adapt}, so the empirical section should be read as
validation of the practical schedule rather than a direct
instantiation of the theorem.

\section{Empirical Evaluation}

\subsection{Setup}
\label{sec:exp_details}

We evaluate on nine public English and Chinese math/STEM benchmarks:
GSM8K \cite{cobbe_training_2021}, MATH
\cite{hendrycks_measuring_2021}, OCW
\cite{lewkowycz_solving_2022}, SAT \cite{azerbayev_llemma_2024},
MMLU$_{\mathrm{STEM}}$ \cite{hendrycks_measuring_2021-1},
CodeContests \cite{li_competition-level_2022}, CMATH
\cite{wei_cmath_2023}, GaokaoMath-Cloze, and GaokaoMath-QA
\cite{zhong_agieval_2024}. We use the standard training and evaluation
split for each dataset and report accuracy on the evaluation split,
except for MMLU$_{\mathrm{STEM}}$, where we follow the common
validation protocol.
Unless otherwise stated, results are averaged over three fixed seeds $\{42,1337,2026\}$.

Rewards are mapped to $[0,1]$ using dataset-specific official evaluation protocols whenever available. For math/QA datasets, reward is exact match after deterministic answer extraction and normalization; when official partial credit is available, we rescale it to $[0,1]$. For MMLU$_{\mathrm{STEM}}$, reward is binary multiple-choice accuracy. For CodeContests, reward is the official test pass rate, with compilation or runtime failures assigned zero.

We start from the public Qwen2.5-14B checkpoint \cite{qwen_qwen25_2025}
and additionally evaluate transfer to Llama-3-8B
\cite{grattafiori_llama_2024} and Gemma-2-9B
\cite{team_gemma_2024}. All sequences use a maximum length of 1024
tokens. We train with AdamW \cite{loshchilov_decoupled_2019}
($( \beta_1,\beta_2)=(0.9,0.95)$, weight decay $0.1$), cosine
learning-rate decay with peak learning rate $1.5\times10^{-5}$,
gradient clipping $1.0$, bf16, DeepSpeed-ZeRO-2
\cite{rajbhandari_zero_2020}, and activation checkpointing on
2$\times$A800 80GB GPUs. Unless otherwise noted, RL/controller
hyper-parameters are $G=8$, $\tau_{\text{base}}=1.0$,
$[\tau_{\min},\tau_{\max}]=[0.5,1.5]$, $\lambda=0.15$,
$\varepsilon_{\mathrm{base}}=0.2$,
$[\varepsilon_{\min},\varepsilon_{\max}]=[0.05,0.4]$, $\beta=0.03$,
$(\alpha,\gamma,\delta,\zeta)=(1.0, 0.5, 0.2, 0.05)$,
$(w_r,w_e,w_k)=(1.0,1.0,0.1)$, $\rho_U=0.95$, and nucleus sampling
$p=0.95$ \cite{holtzman_curious_2020} without top-$k$.

For the main Qwen2.5-14B comparison, each on-policy RL method is
trained for 300M generated continuation tokens on a single node with
2$\times$A800 80GB GPUs. We count generated continuation tokens only,
excluding prompt tokens and evaluation generations, while including
both probe-phase and training-phase generations for fair token-budget
accounting. To reduce contamination risk, we de-duplicate SFT/RL
training data against evaluation prompts, run a low-temperature
memorization probe on dev prompts, and validate transfer on two additional model families.

\subsection{Main Results}
\subsubsection{Evaluation protocol}
Unless otherwise stated, we evaluate on each dataset's standard split
(MMLU$_\mathrm{STEM}$ uses the common validation protocol). Final
scores use greedy decoding unless Maj@$k$ is explicitly reported, in
which case the same $(k,\tau,p)$ is used for all methods
\cite{wang_self-consistency_2023}.

\subsubsection{Baselines and Controls}

To ensure that AGPO’s gains are not merely the result of hyper-parameter choice, we include a broad set of baselines and ablations:

\begin{itemize}
    \item \textbf{PPO} \cite{schulman_proximal_2017}: standard RLHF baseline with a separate value head.
    \item \textbf{Adaptive-KL PPO} \cite{ouyang_training_2022}:
    PPO with an adaptive KL controller, included as a stronger critic-based baseline.
    \item \textbf{DPO} \cite{rafailov_direct_2024}: direct preference optimization without a value model.
    \item \textbf{GRPO (fixed $\varepsilon$)}\cite{shao_deepseekmath_2024}: original GRPO with a constant clipping coefficient.
    \item \textbf{GRPO + ATS}: GRPO plus Adaptive-Temperature Sampling, isolating the impact of ATS alone.
    \item \textbf{AGPO (-ATS)}: Adaptive clipping only (ATS disabled), isolating the impact of adaptive $\varepsilon$.
    \item \textbf{AGPO (full)}: our complete method with both adaptive clipping and ATS.
\end{itemize}
All \emph{on-policy RL} baselines are retrained under the common
training, seed, and token-accounting protocol described in
Sect.~\ref{sec:exp_details}; in particular, AGPO probe tokens are
counted toward the same total budget. DPO is reported as an additional
offline preference-learning baseline and is therefore excluded from
throughput and wall-clock comparisons.

\subsubsection{Performance Comparison}
\begin{table}[ht]
\centering
\caption{Accuracy (\%) on English and Chinese math/STEM benchmarks.
Values are mean accuracies over 3 seeds.
Best in \textbf{bold}; second best in \textit{italics}.}
\label{tab:main_results}
\resizebox{\textwidth}{!}{
\setlength{\tabcolsep}{5.5pt}
\begin{tabular}{l cccccc ccc}
\toprule
 & \multicolumn{6}{c}{\textbf{English Benchmarks}} 
 & \multicolumn{3}{c}{\textbf{Chinese Benchmarks}} \\
\cmidrule(lr){2-7}\cmidrule(lr){8-10}
\textbf{Method} 
& GSM8K & MATH & OCW & SAT & MMLU$_\mathrm{STEM}$ & CodeContests
& CMATH & GaokaoCloze & GaokaoQA \\
\midrule
PPO              
& 54.1 & 32.5 & 8.2  & 76.1 & 48.8 & 12.3 
& 71.2 & 18.9  & 37.7 \\

DPO              
& 57.0 & 33.2 & 9.1  & 77.7 & 46.9 & 12.4 
& 72.5 & 18.8 & 37.5 \\

Adaptive-KL PPO
& 64.8 & 36.9 & 14.6 & 84.3 & 56.0 & 13.3
& 73.4 & 19.6 & 38.5 \\

GRPO (fixed $\varepsilon$) 
& 65.6 & 37.3 & 15.4  & 80.6 & 57.6 & 14.1
& 74.1 & 20.3  & 39.8 \\

GRPO + ATS       
& 66.4 & 38.5 & \textit{16.3}  & 87.9 & 57.1 & 14.9
& 75.9 & 20.5  & 39.2 \\

AGPO (-ATS)      
& \textit{66.8} & \textit{39.2} & 16.2  & \textit{88.6} & \textit{57.9} & \textit{15.2}
& \textit{76.4} & \textit{21.2}  & \textit{40.6} \\

\textbf{AGPO (full)} 
& \textbf{67.3} & \textbf{40.5} & \textbf{17.3}  & \textbf{89.3} & \textbf{58.1} & \textbf{17.3}
& \textbf{77.6} & \textbf{25.9}  & \textbf{41.1} \\
\bottomrule
\end{tabular}
}
\end{table}

\paragraph{Observations.}
GRPO+ATS already improves over GRPO, AGPO (-ATS) further shows the value of adaptive clipping, and full AGPO performs best on all nine benchmarks. Averaged over benchmarks, AGPO gains \textbf{+8.3} pp over PPO and \textbf{+1.4} pp over the strongest ablation.

\subsubsection{Training Time Comparison}

Table~\ref{tab:tps_wallclock} compares (i) total throughput (k tok/s) and (ii) wall-clock hours to 65\% GSM8K-dev across methods.
AGPO incurs negligible overhead relative to GRPO (an $\approx$3\% throughput drop) while remaining \textbf{$\approx$1.5$\times$} faster than PPO in throughput and \textbf{$\approx$1.6$\times$} faster in wall-clock to target accuracy.
These numbers confirm that the critic-free design, even with online statistics, is compute-efficient.

\begin{table}[ht]
\centering
\caption{Throughput and wall-clock comparison on 2 $\times$ A800 80 GB GPUs for on-policy methods only.}
\label{tab:tps_wallclock}
\begin{tabular}{lcc}
\toprule
Method & Total throughput (k tok/s) $\uparrow$ & Hours to 65\% GSM8K $\downarrow$ \\
\midrule
PPO             & 1.72 & 46.2 \\
Adaptive-KL PPO & 1.80 & 42.7 \\
GRPO            & 2.73 & 30.9 \\
GRPO + ATS      & 2.65 & 29.8 \\
\textbf{AGPO (full)} & \textbf{2.64} & \textbf{28.3} \\
\bottomrule
\end{tabular}

\raggedright\footnotesize
\textit{Total throughput aggregates generated tokens per second across all concurrent generation streams on both GPUs; probe tokens are included.}
\end{table}

\subsection{Cross-Family Validation}
\label{sec:crossfamily}

To assess transfer beyond Qwen, we fine-tune Llama-3-8B and Gemma-2-9B under equal token budgets across methods within each backbone. AGPO consistently outperforms fixed-clip GRPO on both GSM8K and MMLU$_\mathrm{STEM}$ across all three backbones, suggesting that the controller is not specific to a single model family.

\begin{table}[t]
\centering
\caption{Cross-family results on GSM8K and MMLU$_\mathrm{STEM}$ (mean$\pm$std over 3 seeds) under equal token budgets that count probe tokens. Avg is the unweighted mean over the six backbone--task cells.}
\label{tab:cross_all}
\small

\begin{tabular}{lccc}
\toprule
\multicolumn{4}{c}{\textbf{GSM8K}} \\
\midrule
\textbf{Method} & \textbf{Llama-3-8B} & \textbf{Qwen2.5-14B} & \textbf{Gemma-2-9B} \\
\midrule
GRPO (fixed $\varepsilon$)
& 63.0$\pm$0.3 & 65.6$\pm$0.3 & 61.8$\pm$0.3 \\
AGPO (-ATS)
& 64.2$\pm$0.3 & 66.8$\pm$0.3 & 63.0$\pm$0.3 \\
AGPO (full)
& \textbf{65.0$\pm$0.3} & \textbf{67.3$\pm$0.3} & \textbf{63.7$\pm$0.3} \\
\bottomrule
\end{tabular}

\vspace{1pt}

\begin{tabular}{lcccc}
\toprule
\multicolumn{5}{c}{\textbf{MMLU$_\mathrm{STEM}$ and Avg}} \\
\midrule
\textbf{Method} & \textbf{Llama-3-8B} & \textbf{Qwen2.5-14B} & \textbf{Gemma-2-9B} & \textbf{Avg (\%)} \\
\midrule
GRPO (fixed $\varepsilon$)
& 55.0$\pm$0.2 & 57.6$\pm$0.2 & 53.2$\pm$0.2 & 59.4 \\
AGPO (-ATS)
& 55.6$\pm$0.2 & 57.9$\pm$0.2 & 53.9$\pm$0.2 & 60.2 \\
AGPO (full)
& \textbf{56.0$\pm$0.2} & \textbf{58.1$\pm$0.2} & \textbf{54.3$\pm$0.2} & \textbf{60.7} \\
\bottomrule
\end{tabular}
\end{table}

\subsection{Ablation Study}
\label{sec:ablation}
\subsubsection{Adaptive Clipping}

We compare fixed-clip GRPO with AGPO (-ATS), isolating the effect of
adaptive clipping. It improves accuracy by +1.2 pp on GSM8K and
+1.9 pp on MATH.

\subsubsection{Adaptive Temperature Sampling}

We analyze ATS in terms of diversity, exploration, and stability.

\begin{figure}[tb]
    \centering
    \includegraphics[width=0.62\linewidth]{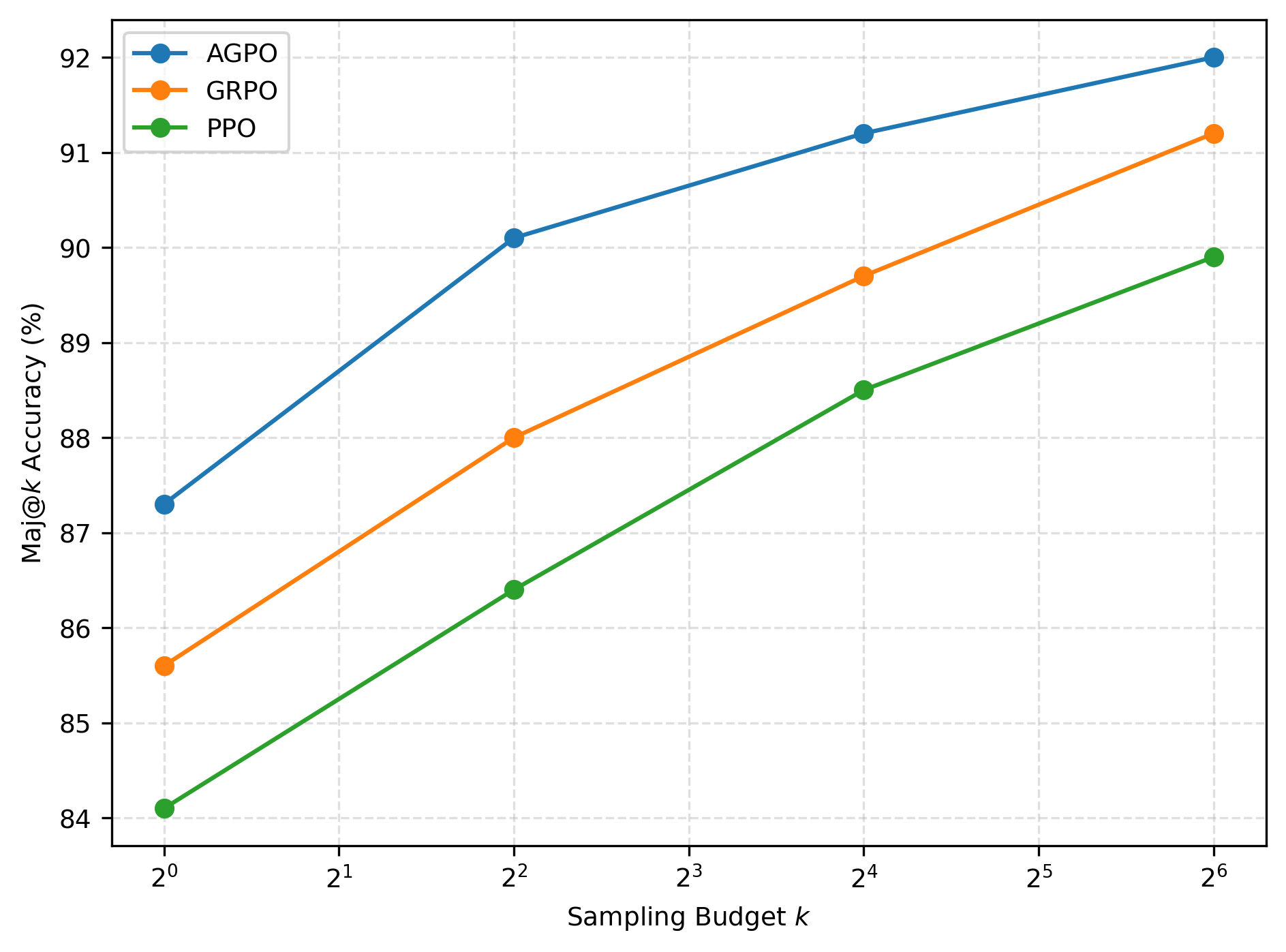}
    \caption{Maj@$k$ exact-match accuracy on GSM8K dev for $k \in \{1,4,16,64\}$. AGPO shows larger gains at intermediate sampling budgets, indicating a better diversity--consistency trade-off.}
    \Description{Line chart showing Maj-at-k accuracy on the GSM8K development set as the sampling budget increases. The horizontal axis lists k values 1, 4, 16, and 64. The vertical axis shows exact-match accuracy. Curves for different methods rise with larger k, indicating that majority voting over more samples improves answer consistency.}
    \label{fig:majk_curve}
\end{figure}

\paragraph{Maj@$k$ curves.}
We evaluate Maj@$k$ on GSM8K dev for $k \in \{1,4,16,64\}$. AGPO matches GRPO at $k{=}1$ and improves more clearly at $k{=}4$ and $16$, indicating better useful diversity. Our AGPO-14B reaches 92.0\% Maj@64 ($\pm 0.3$ pp over 3 seeds).

\paragraph{KL and $\varepsilon_{\text{adaptive}}$ traces.}
Figure~\ref{fig:kl_eps_trace} shows that KL remains bounded while $\varepsilon_t$ adapts downward and eventually reaches the lower bound of $0.05$, indicating stable updates without manual tuning.

\paragraph{Effect of ATS on sample diversity.}
Compared with fixed-temperature GRPO, GRPO+ATS improves Maj@4 from \textbf{87.1\%} to \textbf{88.4\%} on GSM8K and from \textbf{45.0\%} to \textbf{47.8\%} on MATH, with $\pm 0.3$ pp variance across seeds.

Beyond the main setup, we also test leave-one-out ablations and robust dispersion substitutes (MAD/IQR) in Sect.~\ref{sec:robust_ablation}.

\begin{figure}[ht]
    \centering
    \includegraphics[width=0.85\linewidth]{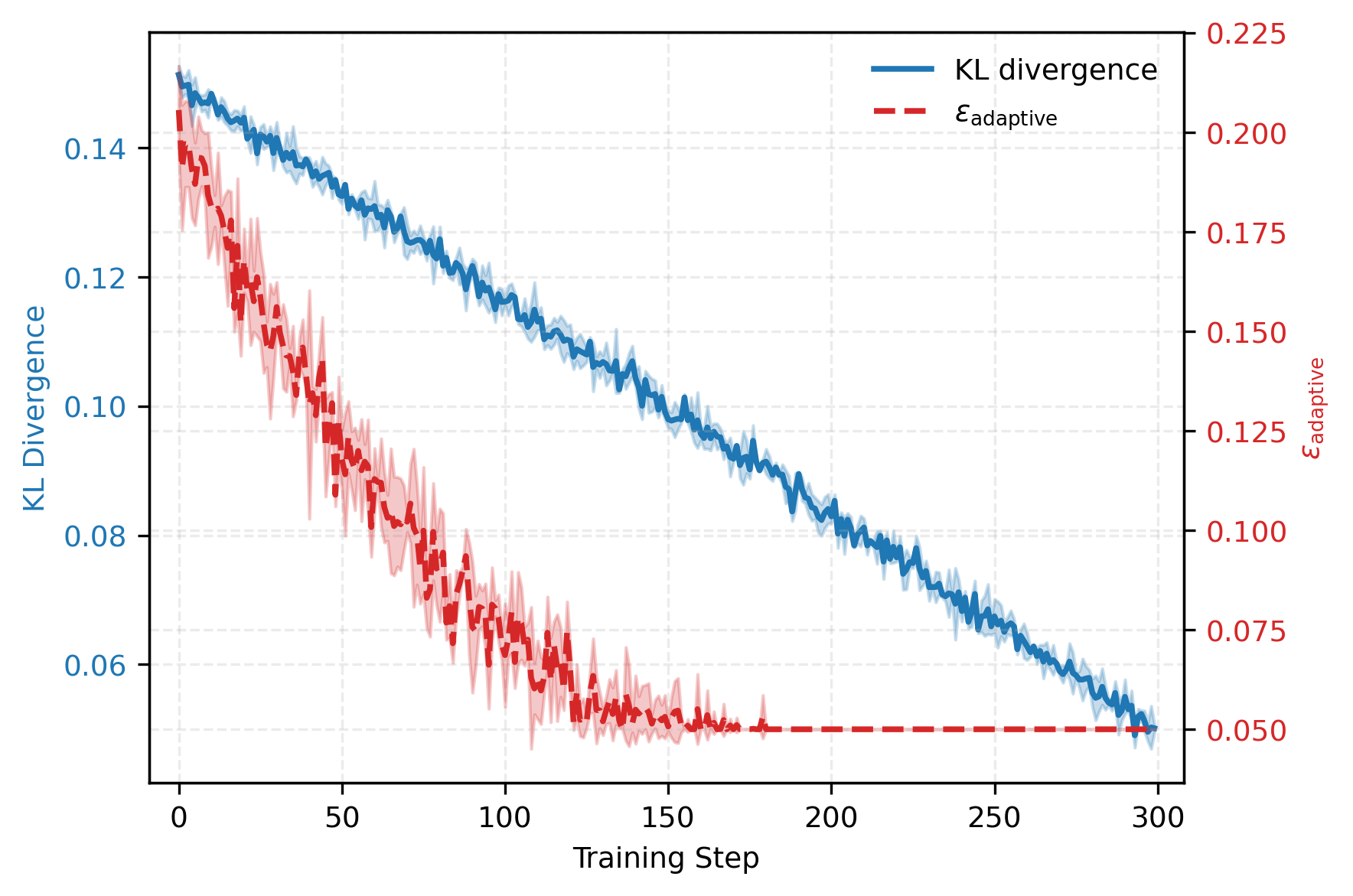}
    \caption{Mean KL divergence and adaptive clip radius $\varepsilon_t$ over 300 training steps (3 seeds). AGPO keeps KL bounded while shrinking $\varepsilon_t$ as training stabilizes.}
    \Description{Line chart with two curves over 300 training steps averaged across three seeds. The blue curve (left y-axis) shows the mean KL divergence gradually decreasing and remaining bounded. The red curve (right y-axis) shows the adaptive clipping coefficient epsilon-t decreasing and eventually saturating near 0.05. Shaded bands around both curves indicate plus or minus one standard deviation.}
    \label{fig:kl_eps_trace}
\end{figure}

\paragraph{ATS behavior.}
Figure~\ref{fig:tau_acc_scatter} plots adaptive temperature $\tau_t$
against batch accuracy. Because $\tau_t$ is driven by centered
uncertainty $\widetilde U_t$, ATS heats early uncertain batches and
cools later, more confident batches, illustrating a shift from
exploration to exploitation.

\begin{figure}[t]
    \centering
    \includegraphics[width=0.72\linewidth]{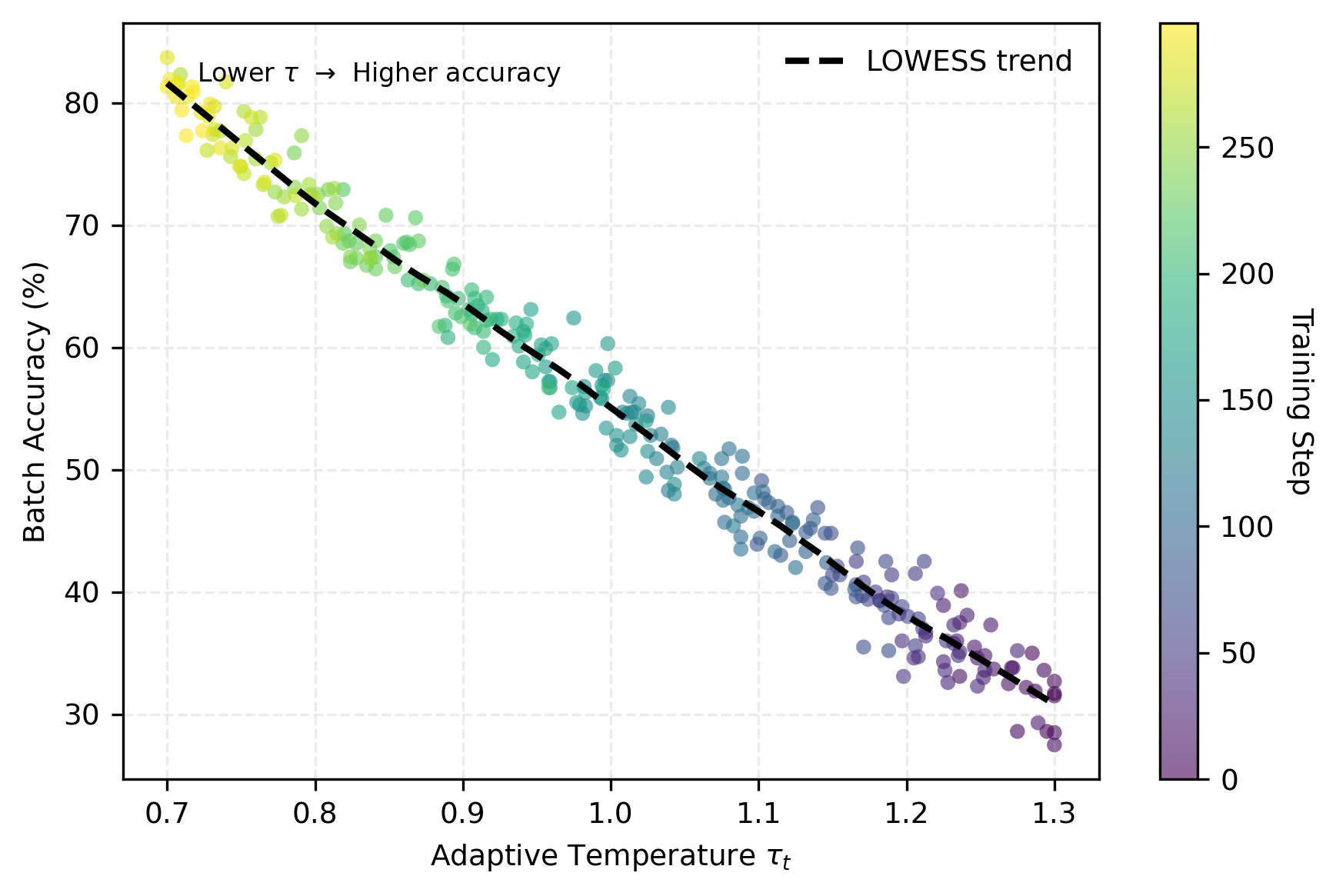}
    \caption{Batch accuracy versus adaptive temperature $\tau_t$ during training. ATS heats early uncertain batches and cools later, more confident batches.}
    \Description{Scatter plot of adaptive temperature versus batch accuracy over training; points transition from dark to light by step index, showing cooling temperature associated with rising accuracy.}
    \label{fig:tau_acc_scatter}
\end{figure}

\paragraph{Phase convention.}
Unless stated otherwise, ATS uses probe rewards to compute
$\hat{\sigma}$ in $U_t$, whereas Eq.~\eqref{eq:adv} uses training-phase
rewards for gradient trajectories; both use the same dispersion
estimator (std/MAD/IQR).

\subsection{Signal Robustness}
\label{sec:robust_ablation}

\paragraph{Motivation.}
AGPO uses probe-derived dispersion, skewness, disagreement, and
step-wise KL. Because $G{=}8$ can make moment estimates noisy, we run
leave-one-out ablations over
$\{\hat{\sigma}, |\tilde{\kappa}_3(r)|, \mathcal{E}_{\text{probe}},
\widehat{D}_{\mathrm{KL}}^{\mathrm{step}}\}$ and test MAD/IQR as
robust substitutes for std:
\[
\mathrm{MAD} = \mathrm{median}_i\,|r_i - \mathrm{median}_j(r_j)|,\quad
\mathrm{IQR} = Q_{0.75}(r) - Q_{0.25}(r),
\]
with
\[
\hat{\sigma}_{\mathrm{MAD}} = 1.4826 \times \mathrm{MAD},\qquad
\hat{\sigma}_{\mathrm{IQR}} = \mathrm{IQR}/1.349.
\]

\paragraph{Implementation.}
We expose two toggles:
\begin{itemize}[leftmargin=1.2em]
\item \texttt{dispersion\_mode} $\in \{\texttt{std}, \texttt{mad}, \texttt{iqr}\}$, controlling $\hat{\sigma}$ used throughout AGPO, including Eqs.~\eqref{eq:adv}, \eqref{eq:uncert_combo}, and \eqref{eq:eps_adapt}.

\item \texttt{signals\_mask} $\in \{0,1\}^4$ for $\{\hat{\sigma}, |\tilde{\kappa}_3|, \mathcal{E}_{\text{probe}}, \widehat{D}_{\mathrm{KL}}^{\mathrm{step}}\}$. For probe-derived signals $\{\hat{\sigma}, |\tilde{\kappa}_3|, \mathcal{E}_{\text{probe}}\}$, masking removes the term from both Eqs.~\eqref{eq:uncert_combo} and \eqref{eq:eps_adapt}; for $\widehat{D}_{\mathrm{KL}}^{\mathrm{step}}$, masking affects Eq.~\eqref{eq:eps_adapt} only.
\end{itemize}

\paragraph{Metrics.}
Besides accuracy, we report training \emph{stability}:
\begin{enumerate}[leftmargin=1.2em,itemsep=1pt,topsep=1pt]
\item Mean KL to reference $\overline{D}_{\mathrm{KL}}(\pi\Vert\pi_{\text{ref}})$.
\item \% of steps hitting clip bounds (\emph{clip saturation rate}).
\item Gradient-norm variance $\mathrm{Var}(\|\nabla\mathcal{L}\|_2)$ (per 1k steps).
\end{enumerate}

\subsubsection{Leave-One-Out Ablations}
\label{sec:loo}

Table~\ref{tab:loo} reports \emph{delta} accuracy (pp) and stability metrics when removing one signal at a time from AGPO (full).
All runs use the same token budget and seeds as the main paper.

\begin{table}[ht]
\centering
\caption{\textbf{Leave-one-out} ablations on key datasets. 
Numbers show \emph{delta} w.r.t.\ AGPO (full); negative is worse for accuracy, positive is worse for stability.}
\label{tab:loo}
\small

\begin{tabular}{lcccc}
\toprule
\multicolumn{5}{c}{\textbf{Accuracy:} $\Delta$ (pp)} \\
\midrule
\textbf{Variant} & \textbf{GSM8K} & \textbf{MATH} & \textbf{MMLU$_\mathrm{STEM}$} & \textbf{CodeContests} \\
\midrule
$-$\,Dispersion $\hat{\sigma}$ & $-0.9$ & $-1.2$ & $-0.4$ & $-0.6$ \\
$-$\,Skewness $|\tilde{\kappa}_3(r)|$ & $-0.3$ & $-0.5$ & $-0.1$ & $-0.2$ \\
$-$\,Probe entropy $\mathcal{E}_{\text{probe}}$ & $-0.7$ & $-0.9$ & $-0.2$ & $-0.5$ \\
$-$\,Step KL $D_{\mathrm{KL}}(\pi_{\text{old}}\Vert \pi_{\text{new}})$ & $-0.6$ & $-0.8$ & $-0.3$ & $-0.4$ \\
\bottomrule
\end{tabular}

\vspace{1pt}

\begin{tabular}{lccc}
\toprule
\multicolumn{4}{c}{\textbf{Stability} (absolute $\Delta$)} \\
\midrule
\textbf{Variant} & $\uparrow\ \overline{D}_{\mathrm{KL}}$ & $\uparrow$ Clip Sat.\% & $\uparrow\ \mathrm{Var}(\|\nabla\|_2)$ \\
\midrule
$-$\,Dispersion $\hat{\sigma}$ & $+0.03$ & $+4.2$ & $+9.5$ \\
$-$\,Skewness $|\kappa_3|$ & $+0.01$ & $+1.3$ & $+3.0$ \\
$-$\,Probe entropy $\mathcal{E}_{\text{probe}}$ & $+0.02$ & $+2.7$ & $+6.0$ \\
$-$\,Step KL $D_{\mathrm{KL}}(\pi_{\text{old}}\Vert \pi_{\text{new}})$ & $+0.05$ & $+5.6$ & $+12.0$ \\
\bottomrule
\end{tabular}
\end{table}

\subsubsection{Robust Dispersion: MAD and IQR}
\label{sec:robust-disp}

We replace $\hat{\sigma}=\mathrm{std}$ by $\hat{\sigma}_{\mathrm{MAD}}$ or $\hat{\sigma}_{\mathrm{IQR}}$ throughout AGPO, namely in Eqs.~\eqref{eq:adv}, \eqref{eq:uncert_combo}, and \eqref{eq:eps_adapt}.
We keep all other settings unchanged.

\begin{table}[ht]
\centering
\caption{Robust dispersion substitutes for $\hat{\sigma}$. 
Accuracy (mean$\pm$std over 3 seeds) and stability metrics (lower is better).}
\label{tab:robust}
\small

\begin{tabular}{lcccc}
\toprule
\multicolumn{5}{c}{\textbf{Accuracy}} \\
\midrule
\textbf{Dispersion} & \textbf{GSM8K} & \textbf{MATH} & \textbf{MMLU$_\mathrm{STEM}$} & \textbf{CodeContests} \\
\midrule
$\mathrm{std}$ (default)   & $67.3\!\pm\!0.3$ & $40.5\!\pm\!0.4$ & $58.1\!\pm\!0.2$ & $17.3\!\pm\!0.3$ \\
$1.4826\times\mathrm{MAD}$ & $67.5\!\pm\!0.3$ & $40.7\!\pm\!0.3$ & $58.2\!\pm\!0.2$ & $17.4\!\pm\!0.2$ \\
$\mathrm{IQR}/1.349$       & $67.4\!\pm\!0.3$ & $40.6\!\pm\!0.3$ & $58.2\!\pm\!0.2$ & $17.4\!\pm\!0.3$ \\
\bottomrule
\end{tabular}

\vspace{1pt}

\begin{tabular}{lccc}
\toprule
\multicolumn{4}{c}{\textbf{Stability} (lower is better)} \\
\midrule
\textbf{Dispersion} & $\overline{D}_{\mathrm{KL}}$ & Clip Sat.\% & $\mathrm{Var}(\|\nabla\|_2)$ \\
\midrule
$\mathrm{std}$ (default)   & $0.80$ & $12.1$ & $1.00$ \\
$1.4826\times\mathrm{MAD}$ & $0.76$ & $9.9$  & $0.88$ \\
$\mathrm{IQR}/1.349$       & $0.77$ & $10.6$ & $0.92$ \\
\bottomrule
\end{tabular}

\raggedright\footnotesize
\textit{Notes.} $\overline{D}_{\mathrm{KL}}$ is in nats; Clip Sat.\% is the percentage of steps with clipped ratio at bounds; $\mathrm{Var}(\|\nabla\|_2)$ is normalized such that the $\mathrm{std}$ row equals $1.00$.
\end{table}

\paragraph{Discussion.}
MAD/IQR are slightly more stable when rewards are binary or heavy-tailed (small $G$), reducing clip saturation and gradient-variance spikes, while remaining close to std on smoother reward distributions.

\FloatBarrier
\section{Conclusion}
\label{sec:conclusion}
We introduced AGPO, a critic-free extension of GRPO that adapts
clipping and sampling from inexpensive group-level statistics. Under
equal generated-token budgets, AGPO improves over PPO/GRPO on nine
English and Chinese math/STEM benchmarks and transfers to Llama-3-8B
and Gemma-2-9B. The main limitation is reliance on reasonably stable
group-level statistics; reward misspecification, reward hacking, and
extremely sparse or adversarially noisy rewards remain open directions.

\FloatBarrier
\bibliographystyle{splncs04}
\bibliography{260320}

\end{document}